\documentclass[a4paper,11pt]{scrartcl}

\RequirePackage{xr}
\RequirePackage[OT1]{fontenc}
\RequirePackage{amsthm,amsmath,amssymb,dsfont,bm}
\RequirePackage{algorithm}
\RequirePackage{algpseudocode}
\RequirePackage{graphicx} 
\RequirePackage[section]{placeins}
\RequirePackage{color}
\RequirePackage{chngcntr}
\RequirePackage{enumerate}
\usepackage[shortlabels]{enumitem}
\RequirePackage{longtable}
\usepackage{booktabs}
\usepackage{array}
\usepackage{multirow}
\usepackage{mathtools}
\usepackage[english]{babel} 
\usepackage{xcolor}

\RequirePackage[OT1]{fontenc} \RequirePackage{amsthm}
\RequirePackage{amsfonts} \RequirePackage{amssymb}
\RequirePackage{bbm}

\RequirePackage{graphicx}
\RequirePackage{verbatim}
\usepackage{url}

\numberwithin{equation}{section}
\RequirePackage{a4wide} 

\theoremstyle{plain}

\newtheorem{theorem}{Theorem}[section]

\def\@bysame#1{\vrule height 1.5pt depth -1pt width 3em \hskip
0.5em\relax}

\newcommand{\N}{\mathbb{N}}
\newcommand{\R}{\mathbb{R}}

\newcommand{\eat}[1]{}

\newcommand{\eins}{\boldmath 1}

\renewcommand{\phi}{{\scriptsize \varphi}}

\newcommand{\Z}{ \mathbb{Z} }

\newcommand{\calA}{\mathcal{A}}

\newcommand{\calE}{\mathcal{E}}

\newcommand{\calG}{\mathcal{G}}

\newcommand{\calX}{\mathcal{X}}

\usepackage{subfigure}
\usepackage{mathtools} 
\usepackage{bm}
\usepackage{dsfont}
\usepackage{color}
\usepackage{amsthm}

\usepackage{float}                 

\begin{document}

\begin{center}
	\begin{minipage}{.8\textwidth}
		\centering 
		\LARGE On Convergence of Moments for Approximating Processes and Applications to Surrogate Models  \\[0.5cm]

		\normalsize
		\textsc{Ansgar Steland}\\[0.1cm]
		Institute of Statistics,\\
		RWTH Aachen University,\\
		Aachen, Germany\\
		Email: \verb+steland@stochastik.rwth-aachen.de+
		
	\end{minipage}
\end{center}

\begin{abstract}
	We study critera for a pair $ (\{ X_n \} $, $ \{ Y_n \}) $ of approximating processes which guarantee closeness of moments by generalizing known results for the special case that $ Y_n = Y $ for all $n$ and $ X_n $ converges to $Y$ in probability. This problem especially arises when working with surrogate models, e.g. to enrich observed data by simulated data, where the surrogates $Y_n$'s are constructed to justify that they approximate the $ X_n $'s.
	The results of this paper deal with sequences of random variables. Since this framework does not cover many applications where surrogate models such as deep neural networks are used to approximate more general stochastic processes, we extend the results to the more general framework of random fields of stochastic processes. This framework especially covers image data and sequences of images. We show that uniform integrability is sufficient, and this holds even for the case of processes provided they satisfy a weak stationarity condition.
\end{abstract}



\section{Introduction}
\label{sec: Introduction}
 
 Suppose we observe a random phenomenon, $ X_n $, $ n \ge 1 $, e.g. representing the outcome of a statistical experiment. Let us further assume that an approximation, $ Y_n $, for $ X_n $ is available, such as a prediction of $ X_n $ based on a (estimated) prediction model or a computer simulation. The later case is receiving increasing interest in the field of uncertainty quantification, where observed data is enriched by data obtained from simulations governed by so-called surrogate models, which are typically obtained from physical knowledge, by design-of-experiment methods, or (non-) parametric estimation from (small) random samples. Examples are linear models, Gaussian processes and deep learning networks. In this case, $ Y_n $ represents the (observable) output of the simulation and $ X_n $ the (unobserved) artificial random variable representing the outcome of the experiment not conducted. In such applications the connection between the true and the surrogate model and therefore between $X_n $ and $ Y_n $ can be rather loose, such that the approximation error can not be analyzed rigorously and assumptions about it have to be made.
 
Assuming that  (uniform) convergence in probability as a minimal requirement holds, the question arises under which conditions the moments of $ Y_n $ are close to the moments of $ X_n $. We study this issue for the case of random variables and the substantially more general framework of random fields of stochastic processes, i.e. families of random variables indexed by a parameter $ \lambda \in \Lambda $ (such as time) and an index $ \bm n $ (such as discrete spatial locations). In this way, the results are general enough to cover various applications including high-dimensional settings and image data.

\section{Criteria for random variable}  

Let us first consider the case of sequences of random variables.

\subsection{Uniform integrability}

Suppose that $ \{ X, X_n : n \ge 1 \} $ is a sequence of random variables defined on a common probability space $ (\Omega, \calA, P) $. Suppose that $ X_n \to X $, as $ n \to \infty $, in probability. Then it is known that the convergence of the moments,
\[
  E(X_n) \to E(X), \qquad n \to \infty,
\]
follows, if $ \{ X_n \} $ is uniformly integrable; this result is usually stated for almost sure convergence, but it holds for convergence in probability as well.  Recall that $ \{ X_n \} $ is called {\em uniformly integrable}, if and only if
\[
  \lim_{A \to \infty} \int_{|X_n| > A} | X_n | \, dP = E[ | X_n | \eins( | X_n | > A ) ] = 0.
\] 
Here and in what follows $ \eins() $ denotes the indicator function. Uniform integrability is equivalent to $ \sup_{n \ge 1 } E | X_n | < \infty $ and
\begin{small}
\begin{equation}
\label{EquivCond2}
  \text{For every $ \varepsilon > 0 $ there exists $ \delta(\varepsilon) > 0 $ such that for any $A \in \calA $: $P(A) < \eta \Rightarrow  \int_A |X_n| \, d P < \varepsilon$ }
\end{equation}
\end{small}
It is well known that the above characterizations are optimal in the sense that if $ X_n $ converges to $ X $ in probability and the $r$th absolute moments, $ E|X_n|^r $, converge to $ E|X|^r $, $ 0 < r < \infty $, then $ \{ X_n \} $ is uniformly integrable, \cite{Chung2001}. Uniform integrability is what is needed to make the step from convergence in probability to convergence of moments. It is worth mentioning that a straightforward way to establish uniform integrability is to verify the sufficient condition
\[
  \sup_{n \ge 1} E |X_n|^{1+\delta} < \infty,
\]
for some $ \delta > 0 $.

There is an interesting relationship to stochastic order relations. Let $ X_1 $ and $ X_2 $ be positive random variables.  $ X_1 $ is less or equal than $ X_2 $ in the {\em increasing convex order}, denoted by \[ X_1 \le_{ic} X_2, \] if 
\[ E \phi(X_1) \le E \phi(X_2) \] for all non-decreasing convex functions $ \phi : [0,\infty) \to [0,\infty) $. Equivalently,
$ H_1(t) \le H_2(t) $ for all $ t \ge 0 $, where $ H_i(t) = \int_t^\infty (1-F_i(u)) \, du $, $ i = 1,2 $, are the integrated survivor functions, see \cite{MuellerStoyan2002}. A sequence $ \{ X_n : n \ge 1 \} $ of random variables is {\em ic-bounded} by a random variable $ Y $, if \[ | X_n | \le_{ic} Y \qquad \text{for all $n \ge 1 $.} \]
In \cite{LV2011} it has been shown that $ \{ X_n : n \ge 1 \} $ is uniformly integrable, if and only if $ \{ X_n \} $ is ic-bounded by an integrable random variable.

\subsection{Convergence of moments for approximations}

As explained in the introduction, it is often not realistic to assume that a given sequence $ \{ X_n \} $ converges to some random variable $ X$,  but instead there exists an approximating sequence $ \{ Y_n : n \ge 1 \} $ of random variables. Then, at best, we can achieve closeness of the moments. Since in many present day real applications the connection between these sequences is somewhat loose in the sense that it is not possible to analyze the accuracy of the approximation rigorously, one has to {\em assume} appropriate conditions. The question arises, whether uniform integrability still suffices to ensure closeness of moments.

The following result shows that the moments of $ Y_n $ are close to the moments of $ X_n $, if $ Y_n $ approximates $ X_n $ and both series are uniformly integrable. Denote $ \| X \|_ r = ( E |X|^r )^{1/r} $ for a random variable $ X $ and $ 0 < r < \infty $.
 
 \begin{theorem} 
\label{ThA}
 	Let $ 0 < r < \infty $.
 	Suppose that $ \{ | X_{n} |^r : n \ge 1 \} $ and $ \{ | Y_{n} |^r : n \ge 1 \} $ are uniformly integrable with \[ | X_{n} - Y_{n} | \stackrel{P}{\to} 0, \] as $ n \to \infty $. Then the following assertions hold.
 	\begin{itemize}
 		\item[(i)] $ E | X_{n} - Y_{n} |^r \to 0 $, as $  n \to \infty $, for $ 0 < r \le 1 $.
 		\item[(ii)] $ | E | X_{n} |^r - E | Y_{n} |^r | \to 0 $, as $ n \to \infty $, for $ 0 < r \le 1 $.
 		\item[(iii)] $ | E(X_{n}) - E( Y_{n} ) | \to 0 $, as $ n \to \infty $, if $ r = 1 $.
 		\item[(iv)] $ | ( E |X_n|^r )^{1/r} - ( E |Y_n|^r )^{1/r} | \to $, as $ n \to \infty$, for $ 0 < r < \infty $.
 	\end{itemize}
 \end{theorem}

Let us consider the following example where we assume concrete models for $ X_n $ and $ Y_n $. Suppose that $ X_n $ is a linear filter processing a random input sequence of $ i.i.d. $ innovations $ \epsilon_t $, $ t \ge 0 $, with mean zero, finite fourth moment and common variance $ \sigma^2 \in (0, \infty) $. Further suppose that $X_n$ is given by an autoregressive process of order 1 with known autoregressive parameter $ \rho \in (-1,1) $, given by
\[
  X_n = \mu + \sum_{j=0}^\infty \rho^j \epsilon_{n-j}
\]
where the mean $ \mu $ is unkown to us. The process $ X_n $, however, can only be observed with a (deterministic) uncertainty $ e_n $, i.e. we have at our disposal the process
\[ 
  X_{n,obs} = X_n + e_n,
\]
where $ e_n $, $n \ge 1 $, is assumed to be a sequence of constants with $ \frac{1}{n} \sum_{i=1}^n e_i \to 0 $, as $ n \to \infty $.  Consider the approximation $ Y_n $ following the surrogate model
\[
  Y_n = \overline{X}_{n,obs} + \sum_{j=0}^{q_n} \rho^j \epsilon_{n-j} 
\]
for some sequence $ q_n $, $ n \ge 1 $, of natural numbers with $ q_n \to \infty $, where
$ \overline{X}_{n,obs} = \frac{1}{n} \sum_{i=1}^n X_{i,obs} $. This means, the surrogate model is obtained
by estimating the unknown mean by the average of the observed data and truncating the infinite sum to obtain a surrogate model from which allows for fast computations. Then 
\[
  | X_n - Y_n | \le \left| \frac{1}{n} \sum_{i=1}^n e_i \right| + | \overline{X}_n - \mu | + \left| \sum_{j>q_n} \rho^j \epsilon_{n-j} \right|. 
\]
The first term on the right-hand is $ o_P(1) $ by virtue of the weak law of large numbers for time series, and the second term can be bounded by
\[
 P\left( \left| \sum_{j>q_n} \rho^j \epsilon_{n-j} \right|  > \varepsilon \right) \le \frac{\sigma^2}{\varepsilon^2} \sum_{j>q_n} |\rho|^{2j} \to 0
 \]
 for any $ \varepsilon > 0 $. Hence,
 \[
   | X_n - Y_n | \stackrel{P}{\to} 0,
 \]
as $ n \to \infty $. Further, by our moment conditions, $ X_n $ and $ Y_n $ are uniformly integrable, and $ \frac{1}{n} \sum_{i=1}^n X_i $ is uniformly integrable, if $ X_n $ has this property, see e.g. \cite{Chung2001}. Hence, the above theorem applies.

 \subsection{Proof}

\begin{proof}[Proof of Theorem~\ref{ThA}]
	By the $ c_r $-inequality (which follows from the inequality $ |x-y|^r \le 2^r ( |x| + |y| ) $ for real numbers $ x, y $ and $ r > 0 $),
	\[
	| X_{n} - Y_{n} |^r \le 2^r(  | X_{n} |^r + | Y_{n} |^r ),
	\]
	we may conclude that $ \{ | X_{n} - Y_{n} |^r : n \ge 1 \} $ is uniformly integrable. Let $ \varepsilon > 0 $.
	We have
	\begin{align*}
	E | X_{n} - Y_{n} |^r & \le \varepsilon^r + E\left[ | X_{n} - Y_{n} |^r \eins\left( | X_{n} - Y_{n} | > \varepsilon \right) \right],
	\end{align*}
	leading to
	\begin{align*}
	\limsup_{n \to \infty}  E | X_{n} - Y_{n} |^r \le \varepsilon^r + \limsup_{n \to \infty}  E\left[ | X_{n} - Y_{\bm n} |^r \eins\left(  | X_{n} - Y_{n} | > \varepsilon \right) \right].
	\end{align*}
	We will show that the second term vanishes. By uniform integrability of $ |X_n - Y_n| $ there exists $ \delta(\varepsilon) > 0 $ such that for any event $A$ with $ P(A) < \delta(\varepsilon) $ we have $ E\left[ | X_{n} - Y_{\bm n} |^r \eins_A \right] < \varepsilon $. 
	Since $  | X_{n} - Y_{n} | \stackrel{P}{\to} 0  $, as $ n \to \infty $, there exists $ n_0 \in \N $ such that for all $ n \ge n_0 $ 
	\[
	P( | X_{n} - Y_{n} | \le \varepsilon ) < \delta(\varepsilon).
	\]
	Therefore, \[  E\left[ | X_{n} - Y_{n} |^r \eins\left(  | X_{n} - Y_{n} | > \varepsilon \right) \right] < \varepsilon.\] Since $ \varepsilon > 0 $ is arbitrary, (i) follows. To show (ii), apply the inequality $ | x + y |^r \le |x|^r + |y|^r $ to obtain
	$| |x|^r - |y|^r | \le |x-y|^r  $, such that $ | E|X_{ n}|^r - E|Y_{ n}|^r | \le E | X_{n} - Y_{n} |^r $, which establishes (ii). (iii) follows by linearity, $ | E(X_{n}) - E(Y_{\bm n}) | \le E | X_{ n} - Y_{ n} | $.
	Lastly, apply Minkowski's inequality to obtain 
	\[
	\| X_n \|_r - \| Y_n \|_r = \| X_n - Y_n + Y_n \|_r - \| Y_n \|_r \le \| X_n - Y_n \|_r
	\]
	and
	\[
	\| Y_n \|_r - \| X_n \|_r = \| Y_n - X_n + X_n \|_r - \| X_n \|_r \le \| Y_n - X_n \|_r 
	\]
\end{proof}

\section{Criteria for random  fields of stochastic processes}

In surrogate modeling applications one often considers models for high-dimensional objects, e.g. stochastic differential equations which need to be solved numerically which can be intractable due to excessive computational costs. Here a surrogate models are constructed which allow for efficient computations and have good approximation properties. Examples are deep learning neural networks and the Gaussian process framework.

Therefore, let us now study a more general framework, namely random fields of stochastic processes, which covers those special cases. 

\subsection{Preliminaries}
 
 Recall that a stochastic process is a family $ \{ X(\lambda) : \lambda \in \Lambda \} $ of random variables $ X(\lambda) $ defined on $ ( \Omega, \calA, P ) $. Here $ \Lambda $ is an arbitrary index set. Such a process is called (strictly) stationary, if the (multivariate) distributions associated to $ \lambda_1, \dots, \lambda_k \in \Lambda $, $ k \in \N $ arbitrary, are shift invariant in the sense that for all $ h $ such that $ \lambda_j + h \in \Lambda $, $ j = 1, \dots, k$, it holds
 \[
   P_{(X(\lambda_1), \dots, X(\lambda_k))} = P_{(X(\lambda_1+h), \dots, X(\lambda_k+h))}.
 \]
 This clearly implies stationarity of the one-dimensional marginal distribution $ P_{X(\lambda)} $, but not vice versa.
  
 A random field of dimension $ q \in \N $ is a family of random elements indexed by a multiindex $ \bm i = (i_1, \dots, i_q)' $ ranging through some set $ I \subseteq \N^q $.  We assume that those random elements attain values in some normed space $E$ with norm $ | \cdot | $.
 
 For example, a two-dimensional random field of random variables indexed by $ \bm i \in \{ 1, \dots, n_1 \} \times \{ 1, \dots, n_2 \} $, i.e. a matrix of dimension $ n_1 \times n_2 $ with random entries, is a natural model for an image of resolution $ n_1 \times n_2 $. 
  
 For two random fields $ \{ X_{\bm n} : \bm n \ge \bm 1 \} $ and $ \{ X_{\bm n} : \bm n \ge \bm 1 \} $, where $ \bm 1 = (1, \dots, 1)' $, the convergence $ | X_{\bm n} - Y_{\bm n} | \stackrel{P}{\to} 0 $, as $ \bm n \to \infty $, is defined as follows: For every $ \varepsilon > 0 $ and every $ \delta > 0 $ there exits $ N \in \N $ such that for all $ n_j \ge N $, $ j = 1, \dots, q $, it holds  $P( | X_{(n_1, \dots, n_q)} - Y_{(n_1,\dots, n_q)} | > \varepsilon ) < \delta $. Limits such as $ | E|X_{\bm n}| - E|Y_{\bm n}|  | \to 0 $, as $\bm n \to \infty $, are defined analogously.
 
 \subsection{Convergence of moments}
 
 Let us assume that we are given a random field of stochastic processes,
 \[
  X_{\bm n}( \lambda ), \ \lambda \in \Lambda, \bm n \in \N^q,
 \]
 which can be approximated by another random field
 \[
	Y_{\bm n}( \lambda ), \ \lambda \in \Lambda, \bm n \in \N^q.
\]

The following theorem shows that the moments of $ Y_{\bm n}( \lambda) $ are uniformly close to the moments of $ X_{\bm n}( \lambda)  $ under weak assumptions, which only concern the (joint) marginal distribution and avoid to assume that $ \sup_{\lambda \in \Lambda} | X_{\bm n} | $ and  $ \sup_{\lambda \in \Lambda} | Y_{\bm n} | $  are uniformly integrable.
 
 \begin{theorem} 
 \label{ThB}
 	Let $ 0 < r < \infty $. Let $ \{ X_{\bm n}( \lambda) : \bm n \ge \bm 1, \lambda \in \Lambda \} $ and $ \{ X_{\bm n}( \lambda) : \bm n \ge \bm 1, \lambda \in \Lambda \} $ be parameterized random fields satisfying the strict marginal stationarity condition
 	\begin{equation}
 		\label{CondStat}
 	(X_{\bm n}( \lambda ), Y_{\bm n}( \lambda ) ) \stackrel{d}{=}  (X_{\bm n}( \lambda' ), Y_{\bm n}( \lambda' ) ) 
 	\end{equation}
 	for all $ \lambda, \lambda' \in \Lambda $ and $ \bm n \ge \bm 1 $.
 	Suppose that $ \{ | X_{\bm n}(\lambda) |^r : \bm n \ge \bm 1 \} $ and $ \{ | Y_{\bm n}(\lambda) |^r : \bm n \ge \bm 1 \} $ are uniformly integrable, $ \lambda \in \Lambda $, with 
 	\[ \sup_{\lambda \in \Lambda}  | X_{\bm n}(\lambda) - Y_{\bm n}(\lambda) | \stackrel{P}{\to} 0, \] 
 	as $ \bm n \to \infty $. Then the following assertions hold.
 	\begin{itemize}
 		\item[(i)] $ sup_{\lambda \in \Lambda} E | X_{\bm n} - Y_{\bm n} |^r \to 0 $, as $ \bm n \to \infty $, for $ 0 < r \le 1 $.
 		\item[(ii)] $ sup_{\lambda \in \Lambda} | E | X_{\bm n} |^r - E | Y_{\bm n} |^r | \to 0 $, as $ \bm n \to \infty $, for $ 0 < r \le 1 $.
 		\item[(iii)] $ sup_{\lambda \in \Lambda} | E(X_{\bm n}) - E( Y_{\bm n} ) | \to 0 $, as $ \bm n \to \infty $, if $ r = 1 $.
 		\item[(iv)] $ sup_{\lambda \in \Lambda}  | \| X_{\bm n} \|_r - \| Y_{\bm n} \|_r | \to 0 $, as $ \bm n \to \infty $, for $ 0 < r < \infty $.
 	\end{itemize}
 \end{theorem}

 The condition (\ref{CondStat}) is automatically satisfied for a large class of cases: Suppose that $ \Lambda = \Z $ and for some sequence of i.i.d. random elements $ \{ \xi_{\lambda} : \lambda \in \Z \} $ taking values in some measurable space $ (E, \calE ) $ we have 
 \[
   X_{\bm n}( \lambda ) = \Psi_{\bm n}( \xi_{\lambda}, \xi_{\lambda-1}, \cdots )
 \]
and
 \[
	Y_{\bm n}( \lambda ) = \Phi_{\bm n}( \xi_{\lambda}, \xi_{\lambda-1}, \cdots )
\]
for all $ \lambda \in \Z $ and $ \bm n \ge \bm 1 $. Then, for arbitrary $ \lambda, \lambda' \in \Lambda $ and all $ \bm n $,
\begin{align*}
  ( X_{\bm n}(\lambda), Y_{\bm n}(\lambda) ) &= ( \Psi_{\bm n}( \xi_{\lambda}, \xi_{\lambda-1}, \cdots ), \Phi_{\bm n}( \xi_{\lambda}, \xi_{\lambda-1}, \cdots ) ) \\
  & \stackrel{d}{=} 
  ( \Psi_{\bm n}( \xi_{\lambda'}, \xi_{\lambda'-1}, \cdots ), \Phi_{\bm n}( \xi_{\lambda'}, \xi_{\lambda'-1}, \cdots ) ) \\
  & =  ( X_{\bm n}(\lambda'), Y_{\bm n}(\lambda') ) 
\end{align*}
which verifies (\ref{CondStat}). Observe that the $ \xi_\lambda $'s may be random variables, random vectors or general random elements such as random functions taking values in an infinite-dimensional space.

Let us now consider parameterized models where
\[
  X_{\bm n}( \lambda ) = X_{\bm n}( \lambda; \vartheta )
\]
for some parameter vector $ \vartheta $. Partition $ \vartheta = ( \eta', \zeta' )' $ and assume that the surrogate model is obtained by estimating, say, $ \zeta $, such that
\[
  Y_{\bm n}( \lambda ) = X_{\bm n}( \lambda; (\eta', \widehat{\zeta}')' ),
\]
where $ \widehat{\zeta}_{\bm n} $ is a statistical estimator of $ \zeta $, obtained by some statistical method of estimation from a random sample (also called calibration to the sample), satisfying
\[
   \| \widehat{\zeta}_{\bm n} - \zeta \| \stackrel{P}{\to} 0,
\]
as $ \bm n \to \infty $, with respect to some norm $ \| \cdot \| $. If the mapping $ X_{\bm n}( \lambda; \vartheta ) $ is
Lipschitz continuous in $ \eta $  with a uniform Lipschitz constant $L$ such that
\[
 \sup_{\lambda \in \Lambda}  | X_{\bm n}( \lambda; (\eta,\zeta_1') ) - X_{\bm n}( \lambda; (\eta, \zeta_2') ) | \le L \| \zeta_1 - \zeta_2 \|,
\]
for all $ \eta, \zeta_1, \zeta_2 $ and all $ \bm n \in \N^q $, then
\[
\sup_{\lambda \in \Lambda}   | X_{\bm n}( \lambda ) - Y_{\bm n}( \lambda ) | \le L \| \widehat{\zeta}_{\bm n} - \zeta \| \stackrel{P}{\to} 0
\]
as $ \bm n \to \infty $, follows.

Putting things together and noting that the $L$ above can be random without affecting the convergence, we obtain the following result.

\begin{theorem} 
\label{ThC}
	Assume that $X_{\bm n}(\lambda) $ and $ Y_{\bm n}(\lambda) $ are parametrized by some parameter $ \vartheta = ( \eta', \zeta' )' \in \Theta $ for some set $ \Theta $, and are of the form
\[
	X_{\bm n}( \lambda ) = \Psi_{\bm n}( \xi_{\lambda}, \xi_{\lambda-1}, \cdots; (\eta', \zeta')'  )
\]
and
\[
	Y_{\bm n}( \lambda ) = \Psi_{\bm n}( \xi_{\lambda}, \xi_{\lambda-1}, \cdots; (\eta', \widehat{\zeta}')'  )
\]
for all $ \lambda \in \Z $ and $ \bm n \in \N^q $, for some sequene $ \{ \xi_{\lambda} : \lambda \in \Lambda \} $, where
$ \lambda \subset \Z $ and $ \xi_{\lambda} $ are i.i.d. and attain values in some measurable space $ (E, \calE) $. 
Further suppose that the mapping $ \Psi_{\bm n} $ is Lipschitz continuous in $ \zeta $ in the sense that for some random variable $ L $
\[
	\sup_{\lambda \in \Lambda}  | X_{\bm n}( \lambda  ) - Y_{\bm n}( \lambda  ) | \le L \| \zeta_1 - \zeta_2 \|,
\]
for all $ \eta, \zeta_1, \zeta_2 $ and all $ \bm n \in \N^q $. If $ \widehat{\zeta} $ is a consistent estimator of $ \zeta $, then
the assumptions of Theorem~\ref{ThB} are satisfied.
\end{theorem}

\section{Applications to surrogate models: Deep Learning and Gaussian Processes}

As a surrogate model is used to generate cheap artificial data sets (e.g. to enrich real observed data), classes of models with convincing approximation properties are preferable.
Deep learning neural networks as well as Gaussian processes are two widespread frameworks for surrogate modeling, as they satisfy this requirement. Typically, one calibrates such a model to a (relatively small) data set of real data and then simulates from the fitted model to obtain simulated data samples which should be close to real samples.

\vskip 0.2cm
\noindent
\textbf{Deep learning networks}

A deep learning  artificial neural network, see e.g. \cite{Goodfellow-et-al-2016}, is a mapping $ f : \bm \calX \to \R^q $, which maps an input vector $ \bm x $ of the input space $ \bm \calX \subset \R^p $, assumed to be a compact set, to a $q$-dimensional output vector $ \bm y $, $ p, q \in \N $, and is given by the composition of $H$ layers in the form 
\[
  \bm y = f( \bm x ) = f_H( \cdots f_2( \bm W_2 f_1( \bm W_1 \bm x + \bm b_1 ) + \bm b_2) \cdots ), \qquad \bm x \in \bm \calX,
\]
where $ \bm W_l $ are weighting matrices, $\bm b_l $ intercept terms and $ f_l $ are activation functions, $ l = 1, \dots, H$. The parameter vector of the net is $ \vartheta = ( \text{vec} \bm W_1, \dots, \text{vec} \bm W_H, \bm b_1', \dots, \bm b_H' )' $, where $ \text{vec} \bm A $ denotes the vectorized version of a matrix $ \bm A $ obtained by stacking columns, such that
\[
	f( \bm x ) = f( \bm x; \vartheta ).
\]
The activation functions are typically nonlinear and always chosen as Lipschitz continuous functions. Clearly, the sum of two Lipschitz functions with Lipschitz constants $L_1 $ and $ L_2 $ is again Lipschitz with Lipschitz constant $ L_1 + L_1 $,  and  the composition $  f \circ g $ of two Lipschitz functions $f $ and $g$ with constants $ L_1 $ and $ L_2 $ is again Lipschitz with Lipschitz constant $ L_1 L_2 $, because $ | f(g(x)) - f(g(y)) | \le L_1 | g(x) - g(y) | \le L_1 L_2 |x-y| $. Therefore,
such deep learning networks are Lipschitz continuous in the parameters. Indeed, current efforts focus on calculating the Lipschitz constants. It is not restrictive to assume that $ E | f( \bm X; \vartheta ) |^{1+\delta} < \infty $, for some $ \delta > 0 $, where $ \bm X $ is a random input. Alternatively, assume that $ E | \bm X |^{1+\delta} < \infty $ holds and the existence of some $ \bm x_0 \in \bm \calX $ such that $ f(\bm x_0; \vartheta ) = 0 $. Then
\begin{align*}
  E| f( \bm X; \vartheta ) |^{1+\delta} 
  & = E| f( \bm X; \vartheta ) - f( \bm x_0; \vartheta) |^{1+\delta} \\
  & \le L E | \bm X - \bm x_0|^{1+\delta} \\
  & \le L ( \| \bm X \|_{1+\delta} + \| \bm x_0 \|_{1+\delta} )^{1+\delta}  \\
  & < \infty,
\end{align*}
where $L$ denotes the Lipschitz constant of the net.

If a deep learning network is trained at time $n $, say from a data stream, using the most recent $n$ data points $ \bm X_1, \dots, \bm X_{n} $ with associated outputs $ \bm Y_{1}, \dots, \bm Y_{n} $,  by estimating the parameters $ \vartheta $, 
the trained network is given by
\[
  f( \bm x; \widehat{\vartheta}_n )
\]
where
\[ 
	\widehat{\vartheta}_n = \widehat{\vartheta}_n( \xi_1,  \dots, \xi_{n} ),
\] 
with $ \xi_i = ( \bm X_i', \bm Y_i' )' $, $ i = 1, \dots, n $. If one now simulates an input $ \bm X \sim G $, say for some $G$ with $ \int |x|^{1+\delta} \, d G(x) < \infty $, then the associated  output,
\[
  Y_n = f( \bm X; \widehat{\vartheta}_m( \xi_1,  \dots,  \xi_{n} )),
\]
is a surrogate for $ X_n = f( \bm X; \vartheta ) $. Consequently, Theorem~\ref{ThC} applies.

\vskip 0.2cm
\noindent
\textbf{Gaussian processes kriging}

Let us consider the following example studied in some depth in \cite{DSD2013} dealing with reliability analysis. Let $ X $ denote a $d$-dimensional random vector with density $ f_X $ and support $D$. Given a performance function $g : D \to \R $ a failure, e.g. of a system, can be modeled by the event $ \{ g(X) \le 0 \} $. Conducting such experiments in practice is, however, sometimes expensive, whereas simulations from an appropriate (surrogate) model are usually cheap. Since $g$ is unknown, a surrogate model for $g$ is used, which allows to estimate (or predict) $g$ and quantify the involved uncertainty. The Gaussian process kriging approach assumes that $g$ is a sample path of an underlying Gaussian process $ \calG $,
\[
  \calG(x) = f(x)'\beta + Z(x), \qquad x \in D.
\]
Here $ f(x)'\beta $ is the linear predictor with respect to given functions $ f_1(x), \dots, f_p(x) $ from a basis of, say, the function space $L_2 $, and a parameter vector $ \beta \in \R^p $, and $ Z(x) $ is a mean zero stationary Gaussian process with a stationary correlation function, often chosen in practice as 
\[
  R(x-x', \ell_1, \dots, \ell_d ) = \exp\left( - \sum_{k=1}^d [(x_k-x_k')/ \ell_l]^2 \right),
\]
for scaling parameters $ \ell_1, \dots, \ell_p > 0$. Given a random sample $ X_1, \dots, X_n $ of size $n$, the best linear unbiased (kriging) estimator of $ \calG(x) $ at $x$ is Gaussian and interpolates the observations $ g(X_i) = f(X_i)'\beta $, if $ g \in \text{span}\{ f_1, \dots, f_p \} $, i.e. there is no residual uncertainty (at the observed data points). By increasing $p$ as $n$ gets larger, any $L_2$-function can be estimated in this way. Note that the predictor depends on $n$. An observation $ Y_n = Y_n(x) $ simulated from the surrogate model for some $ x \not\in \{ X_1, \dots, X_n \} $ is regarded as an approximation of an unobserved  $ X_n(x) $ (obtained in a Gedankenexperiment which is too expensive to be carried out).

 \subsection{Proof}

\begin{proof}[Proof of Theorem~\ref{ThB}]
	By the $ c_r $-inequality
	\[
	| X_{\bm n}(\lambda) - Y_{\bm n}(\lambda) |^r \le 2^r( | X_{\bm n}(\lambda) |^r + | Y_{\bm n}(\lambda) |^r ),
	\]
	we may conclude that $ \{ | X_{\bm n}(\lambda) - Y_{\bm n}(\lambda) |^r : \bm n \ge \bm 1 \} $ is uniformly integrable. Let $ \varepsilon > 0 $.
	We have
	\begin{align*}
	E | X_{\bm n}(\lambda) - Y_{\bm n}(\lambda) |^r & \le \varepsilon^r + E\left[ | X_{\bm n}(\lambda) - Y_{\bm n}(\lambda) |^r \eins\left( | X_{\bm n}(\lambda) - Y_{\bm n}(\lambda) | > \varepsilon \right) \right].
	\end{align*}
	Fix $ \lambda_0 \in \Lambda $.
	By uniform integrability, there exists $ \eta = \eta( \lambda_0 ) > 0 $ such that for all events $A$ with $ P(A) < \eta $ we have
	$ E[ | X_{\bm n}(\lambda_0) - Y_{\bm n}(\lambda_0) |^r \eins_A  ] < \varepsilon $. 	
	Since $  \sup_{\lambda \in \Lambda} | X_{\bm n}(\lambda) - Y_{\bm n}(\lambda) | \stackrel{P}{\to} 0  $, as $ \bm n \to \infty $, there exists $ \bm n_0 $ such that for all $ \bm n \ge \bm n_0 $ 
	\[ 
	P( | X_{\bm n}(\lambda_0) - Y_{\bm n}(\lambda_0) | > \varepsilon ) \le P\left( \sup_{\lambda \in \Lambda} | X_{\bm n}(\lambda) - Y_{\bm n}(\lambda) | > \varepsilon \right) < \eta.
	\]
	It follows that
	\begin{align*}
	&\sup_{\lambda \in \Lambda}  E[ | X_{\bm n}(\lambda) - Y_{\bm n}(\lambda) |^r \eins( | X_{\bm n}(\lambda) - Y_{\bm n}(\lambda) | > \varepsilon )  ] \\ & \qquad =
	E[ | X_{\bm n}(\lambda_0) - Y_{\bm n}(\lambda_0) |^r \eins( | X_{\bm n}(\lambda_0) - Y_{\bm n}(\lambda_0) | > \varepsilon )  ] \\ 
	& \qquad <  \varepsilon,
	\end{align*}
	where the equality is a consequence of (\ref{CondStat}), leading to 
	\[
	\sup_{\lambda \in \Lambda} E | X_{\bm n}(\lambda) - Y_{\bm n}(\lambda) |^r  \le \varepsilon^r + \varepsilon, \qquad \bm n \ge \bm n_0,
	\]
	which shows (i), since $ \varepsilon $ is arbitrary.
	To show (ii), apply the inequality $ | x + y |^r \le |x|^r + |y|^r $ to obtain
	$| |x|^r - |y|^r | \le |x-y|^r $, such that 
	\[ | E|X_{\bm n}(\lambda)|^r - E|Y_{\bm n}(\lambda)|^r | \le E | X_{\bm n}(\lambda) - Y_{\bm n}(\lambda) |^r \le \sup_{\lambda \in \Lambda} E | X_{\bm n}(\lambda) - Y_{\bm n}(\lambda) |^r , \] 
	which yields
	which establishes (ii). (iii) follows by linearity, $ \sup_{\lambda \in \Lambda}  | E(X_{\bm n}) - E(Y_{\bm n}) | \le \sup_{\lambda \in \Lambda}  E | X_{\bm n} - Y_{\bm n} | $, and (iv) is shown as in the proof of Theorem~\ref{ThA}.
\end{proof}


\end{document}